\documentclass{article}

\usepackage{times}
\usepackage{graphicx} % more modern

\usepackage{subcaption}

\usepackage{natbib}

\usepackage[algo2e,ruled]{algorithm2e}

\usepackage{hyperref}

\usepackage[accepted]{icml2016}

\usepackage{eqnarray,amsmath}
\usepackage{amssymb}

\usepackage{amsthm}

\usepackage{paralist}
\newcommand{\visname}{{\sc MDPv\/is}}
\newcommand{\algname}{MFMCi}
\DeclareMathOperator*{\argmin}{\arg\!\min}

\newtheorem{theorem}{Theorem}
\newtheorem{corollary}{Corollary}

\theoremstyle{definition}
\newtheorem{definition}{Definition}[section]

\icmltitlerunning{Factoring Exogenous State for MFMC}

\begin{document} 

\twocolumn[
\icmltitle{Factoring Exogenous State for Model-Free Monte Carlo}

\icmlauthor{Sean McGregor}{arXiv@seanbmcgregor.com}
\icmladdress{School of Electrical Engineering and Computer Science,
            Oregon State University}
\icmlauthor{Rachel Houtman}{rachel.houtman@oregonstate.edu}
\icmladdress{College of Forestry,
            Oregon State University}
\icmlauthor{Claire Montgomery}{claire.montgomery@oregonstate.edu}
\icmladdress{College of Forestry,
            Oregon State University}
\icmlauthor{Ronald Metoyer}{rmetoyer@nd.edu}
\icmladdress{College of Engineering,
            University of Notre Dame}
\icmlauthor{Thomas G. Dietterich}{tgd@oregonstate.edu}
\icmladdress{School of Electrical Engineering and Computer Science,
            Oregon State University}

\icmlkeywords{Markov Decision Processes, Monte Carlo Methods, Surrogate Modeling, Visualization}

\vskip 0.3in
]

%\title{Factoring Exogenous State for Model-Free Monte Carlo}
%\author{Sean McGregor,$^{0}$ Rachel Houtman,$^{1}$ Hailey Buckingham,$^{1}$\\
%Claire Montgomery,$^{1}$ Ronald Metoyer,$^{2}$ and Thomas G.~Dietterich$^{0}$\\
%School of Electrical Engineering and Computer Science, Oregon State University$^{0}$\\
%College of Forestry, Oregon State University$^{1}$\\
%Department of Computer Science and Engineering, University of Notre Dame$^{1}$\\
%\texttt{AAAI17@seanbmcgregor.com}}

\begin{abstract}
  Policy analysts wish to visualize a range of policies for large
  simulator-defined Markov Decision Processes (MDPs). One visualization approach is to invoke the
  simulator to generate on-policy trajectories and then visualize
  those trajectories. When the simulator is expensive, this is not
  practical, and some method is required for generating trajectories
  for new policies without invoking the simulator. The method of
  Model-Free Monte Carlo (MFMC) can do this by stitching together
  state transitions for a new policy based on previously-sampled
  trajectories from other policies. This ``off-policy Monte Carlo
  simulation'' method works well when the state space has low
  dimension but fails as the dimension grows. This paper describes a
  method for factoring out some of the state and action variables so
  that MFMC can work in high-dimensional MDPs. The new
  method, \algname, is evaluated on a very challenging wildfire
  management MDP.
\end{abstract}

\section{Introduction}

% todo: this saves space?
%\captionsetup{belowskip=2pt,aboveskip=2pt,labelfont=bf}
%\setlength{\textfloatsep}{0pt}% Remove \textfloatsep
%\captionsetup[subfigure]{skip=-1pt}

\renewcommand{\dblfloatpagefraction}{0.9}
\renewcommand{\floatpagefraction}{0.9}

As reinforcement learning systems are increasingly deployed
in the real-world, methods for justifying their ecological
validity becomes increasingly important.
For example, consider the problem of wildfire
management in which land managers must decide when and where to fight
fires on public lands. Our goal is to create an interactive
visualization environment in which policy analysts can define various
fire management polices and evaluate them through comparative
visualizations. The transition dynamics of our fire management MDP are
defined by a simulator that takes as input a detailed map of the
landscape, an ignition location, a stream of weather conditions, and a
fire fighting decision (i.e., suppress the fire vs. allow it to burn),
and produces as output the resulting landscape map and associated
variables (fire duration, area burned, timber value lost, fire
fighting cost, etc.).  The simulator also models the year-to-year
growth of the trees and accumulation of fuels.  Unfortunately, this
simulator is extremely expensive. It can take up to 7 hours to
simulate a single 100-year trajectory of fire ignitions and resulting
landscapes. How can we support interactive policy analysis when the
simulator is so expensive?

Our approach is to develop a surrogate model that can substitute for
the simulator. We start by designing a small set of ``seed policies''
and invoking the slow simulator to generate several 100-year
trajectories for each policy. This gives us a database of state
transitions of the form $(s_t,a_t,r_t,s_{t+1})$, where $s_t$ is the
state at time $t$, $a_t$ is the selected action, $r_t$ is the
resulting reward, and $s_{t+1}$ is the resulting state. Given a new
policy $\pi$ to visualize, we apply the method of Model-Free Monte
Carlo (MFMC) developed by \citet{Fonteneau2013}
to simulate trajectories for $\pi$ by stitching together state
transitions according to a given distance metric $\Delta$. Given a
current state $s$ and desired action $a=\pi(s)$, MFMC searches the
database to find a tuple $(\tilde{s}, \tilde{a}, r,s')$ that minimizes
the distance $\Delta((s,a), (\tilde{s}, \tilde{a}))$.  It then uses
$s'$ as the resulting state and $r$ as the corresponding one-step
reward. We call this operation ``stitching'' $(s,a)$ to $(\tilde{s},\tilde{a})$. MFMC is guaranteed to give reasonable simulated trajectories under assumptions about the smoothness of the transition dynamics and reward function and provided that each matched tuple is removed from the database when it is used. Algorithm \ref{alg:mfmc} provides the pseudocode for MFMC generating a single trajectory. 

\citet{Fonteneau2010c} apply MFMC to estimate the expected cumulative return of a new policy $\pi$ by calling MFMC $n$ times and computing the average cumulative reward of the resulting trajectories. We will refer to this as the MFMC estimate $V_{MFMC}^\pi(s_0)$ of $V^\pi(s_0)$. 

\begin{algorithm2e}
%\SetKwProg{Fn}{Function}{}{}
%\DontPrintSemicolon
\textbf{Input Parameters:} Policy $\pi$, horizon $h$, starting state $s_0$, distance metric $\Delta(.,.)$, database $D$\;
\textbf{Returns: $(s,a,r)_1,...,(s',a',r')_h$}\;
$t \gets \emptyset$\;
$s\gets s_0$\;
\While{length($t$)$<h$}{
  $a\gets \pi(s)$\;
  $\mathcal{H} \gets \underset{(\overset{\sim}{s}, \overset{\sim}{a}, r, s') \in D }{\argmin}\Delta ((s,a), (\overset{\sim}{s},\overset{\sim}{a}))$\;
  $r\gets \mathcal{H}^r$\;
  append$(t,(s, a, r))$\;
  $s\gets \mathcal{H}^{s'}$\;
  $D \gets D\setminus \mathcal{H}$\;
}
return($t$)\;
\caption{
  MFMC for a single trajectory.
  When generating multiple trajectories for a single
  trajectory set, you cannot re-use state transitions from $D$
  \cite{Fonteneau2010c}.}
\label{alg:mfmc}
\end{algorithm2e}

In high-dimensional spaces (i.e., where the states and actions are
described by many features), MFMC breaks because of two related
problems. First, distances become less informative in high-dimensional
spaces. Second, the required number of seed-policy trajectories grows
exponentially in the dimensionality of the space.  The main technical
contribution of this paper is to introduce a modified
algorithm, \algname, that reduces the dimensionality of the distance
matching process by factoring out certain exogenous state variables
and removing the features describing the action.  In many
applications, this can very substantially reduce the dimensionality of
the matching process to the point that MFMC is again practical.

This paper is organized as follows.  First, we briefly review previous
research in surrogate modeling.  Second, we introduce our method for
factoring out exogenous variables. The method requires a modification
to the way that trajectories are generated from the seed
policies. With this modification, we prove that \algname\ gives sound
results and that it has lower bias and variance than MFMC.  Third, we
conduct an experimental evaluation of \algname\ on our fire management
problem. We show that \algname\ gives good performance for three
different classes of policies and that for a fixed database size, it
gives much more accurate visualizations.

\section{Related Work} \label{sec:related}

Surrogate modeling is the construction of a fast simulator that can substitute for a slow simulator.  When designing a surrogate model for our wildfire suppression problem, we can consider several possible approaches.

First, we could write our own simulator for fire spread, timber harvest, weather, and vegetative growth that computes the state transitions more efficiently.  For instance, Arca et al.~\yrcite{Arca2013} use a custom-built model running on GPUs to calculate fire risk maps and mitigation strategies.  However, developing a new simulator requires additional work to design, implement, and (especially) validate the simulator.  This cost can easily overwhelm the resulting time savings.

A second approach would be to learn a parametric surrogate model from data generated by the slow simulator.  For instance, Abbeel et al.~\yrcite{Abbeel2005} learn helicopter dynamics by updating the parameters of a function designed specifically for helicopter flight.  Designing a suitable parametric model that can capture weather, vegetation, fire spread, and the effect of fire suppression would require a major modeling effort.

Instead of pursuing these two approaches, we adopted the method of Model-Free Monte Carlo (MFMC). In MFMC, the model is replaced by a database of transitions computed from the slow simulator.  MFMC is ``model-free'' in the sense that it does not learn an explicit model of the transition probabilities. In effect, the database constitutes the transition model (c.f., Dyna; \cite{s-ialprbadp-90}).

\section{Notation} \label{sec:proof}

We work with the standard finite horizon undiscounted MDP \cite{Bellman1957,Puterman1994}, denoted by the tuple ${\cal M}=\langle{S,A,P,R,s_0,h\rangle}$. $S$ is a finite set of states of the world; $A$ is a finite set of possible actions that can be taken in each state; $P: S\times A\times S \mapsto [0,1]$ is the conditional probability of entering state $s'$ when action $a$ is executed in state $s$; $R(s,a)$ is the finite reward received after performing action $a$ in state $s$; $s_0$ is the starting state; and $\pi: S \mapsto A$ is the policy function that selects which action $a\in A$ to execute in state $s\in S$. We additionally define $D$ as the state transition database.  

In this paper, we focus on two queries about a given MDP. First, given a policy $\pi$, we wish to estimate the expected cumulative reward of executing that policy starting in state $s_0$: $V^\pi(s_0) = \mathbb{E}[\sum_{t=0}^{h} R(s_t,\pi(s_t)) | s_0, \pi]$. Second, we are interested in visualizing the distribution of the states visited at time $t$: $P(s_t|s_0, \pi)$. In particular, let $v^1,\ldots,v^m$ be functions that compute interesting properties of a state. For example, in our fire domain, $v^1(s)$ might compute the total area of old growth Douglas fir and $v^2(s)$ might compute the total volume of harvestable wood. Visualizing the distribution of these properties over time gives policy makers insight into how the system will evolve when it is controlled by policy $\pi$. 

\section{Factoring State to Improve MFMC}

\begin{figure}
    \centering
    \begin{subfigure}[t]{\columnwidth}
      \centering
        \includegraphics[width=0.60\columnwidth]{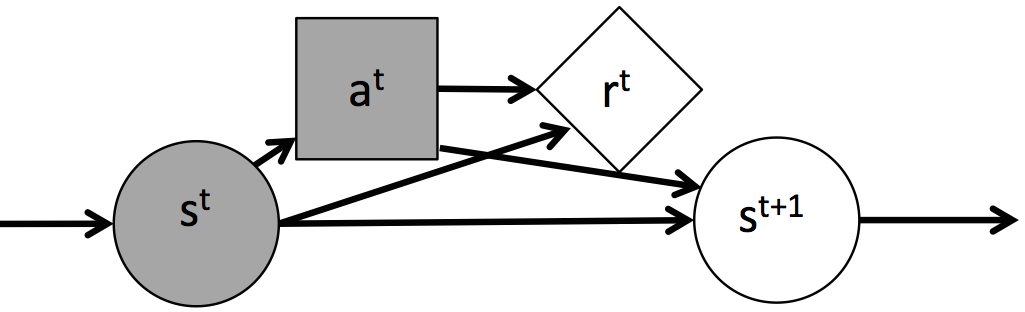}
        \caption{The standard MDP transition.}
        \label{fig:MDP1}
    \end{subfigure}
    ~~~
    \begin{subfigure}[t]{\columnwidth}
      \centering
        \includegraphics[width=0.60\columnwidth]{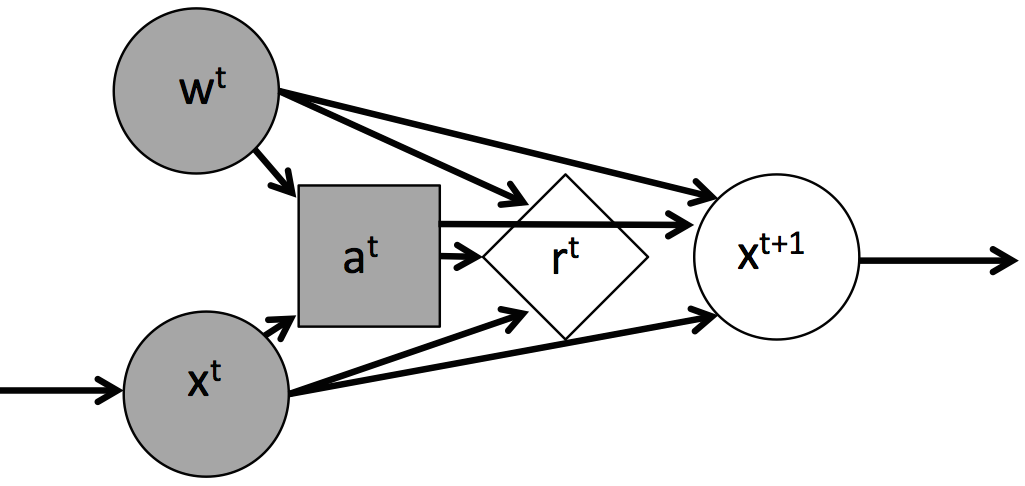}
        \caption{
        MDP transition with \emph{exogenous} ($w$) and \emph{Markovian}
        variables ($x$).}
        \label{fig:MDP3}
    \end{subfigure}
    \caption{MDP probabilistic graphical models.}
    \label{fig:PGMMDP}
\end{figure}

We now describe how we can factor the state variables of an MDP in order to reduce the dimensionality of the MFMC stitching computation.  State variables can be divided into Markovian and Time-Independent random variables.  A time-independent random variable $x_{t}$ is exchangeable over time $t$ and does not depend on any other random variable (including its own previous values). A (first-order) Markovian random variable $x_{t+1}$ depends on its value $x_t$ at the previous time step. In particular, the state variable $s_{t+1}$ depends on $s_t$ and the chosen action $a_t$. Variables can also be classified as endogenous and exogenous.
%The variable $x_{t+1}$ is exogenous if its distribution neither depends on $a_t$ nor on $s_t$.
The variable $x_t$ is exogenous if its distribution is independent of $a_{t'}$ and $s_{t'} \setminus \{x_t'\}$ for all $t'\leq t$.
Non-exogenous variables are endogenous.
The key insight of this paper is that if a variable is time-independent and exogenous, then it can be removed from the MFMC stitching calculation as follows.

Let us factor the MDP state $s$ into two vectors of random variables: $w$, which contains the time-independent, exogenous state variables and $x$, which contains all of the other state variables (see Figure \ref{fig:PGMMDP}).  In our wildfire suppression domain, the state of the trees from one time step to another is Markovian, but our policy decisions also depend on exogenous weather events such as rain, wind, and lightning.

We can formalize this factorization as follows.

\theoremstyle{definition}
\begin{definition}{A Factored Exogenous MDP}
is an MDP such that the state $(x,w)$ and next state $(x',w')$ are
related according to

\begin{equation} \label{eq:fmdp__exo_eqn}
 Pr(x',w'|x,w,a)= Pr(w')Pr(x'|x,w,a).
\end{equation}
\end{definition}

This factorization allows us to avoid computing similarity in the complete state space $s$. Instead we only need to compute the similarity of the Markov state $x$.  Without the factorization, MFMC stitches $(s,a)$ to the $(\tilde{s},\tilde{a})$ in the database $D$ that minimizes a distance metric $\Delta$, where $\Delta$ has the form $\Delta((s,a),(\tilde{s},\tilde{a}))\mapsto\mathbb{R}^{+}$.  Our new algorithm, \algname{}, makes its stitching decisions using only the Markov state. It stitches the current state $x$ by finding the tuple $(\tilde{x}, \tilde{w}, a, r, x')$ that minimizes the lower-dimensional distance metric $\Delta_i(x,\tilde{x})$.  MFMCi then adopts $(\tilde{x},\tilde{w})$ as the current state, computes the policy action $\tilde{a} = \pi(\tilde{x},\tilde{w})$, and then makes a transition to $x'$ with reward $r$.  The rationale for replacing $x$ by $\tilde{x}$ is the same as in MFMC, namely that it is the nearest state from the database $D$. The rationale for replacing $w$ by $\tilde{w}$ is that both $w$ and $\tilde{w}$ are exchangeable draws from the exogenous time-independent distribution $P(w)$, so they can be swapped without changing the distribution of simulated paths.

\begin{algorithm2e}[t]
\SetKwProg{Fn}{Function}{}{}
\DontPrintSemicolon
\textbf{Input Parameters:} Policy $\pi$, horizon $h$, trajectory count $n$, transition simulator $f_x$, reward simulator $f_r$, exogenous distribution $P(w)$, stochasticity distribution $P(z)$\;
\textbf{Returns:} $nh$ transition sets $B$\;
$D \gets \emptyset$\;
\While{$|D|<nh$}{
  $x=f_{x}(\cdot,\cdot,\cdot,\cdot)$ // Draw initial Markov state\;
  $l=0$\;
  \While{$l<h$}{
    $B \gets \emptyset$\;
    $w\sim P(w)$\;
    $z\sim P(z)$\;
    \For{$a\in A$} {
      $r \gets f_r(x,a,w,z)$\;
      $x' \gets f_{x}(x,a,w,z)$\;
      $B \gets B \cup \{(x,w,a,r,x')\}$\;
    }
    append($D$,$B$)\;
    $x \gets B_{\pi(x,w)}^{x'}$\;
    $l \gets l+1$\;
  }
}
return($D$)\;
\caption{
Populating $D$ for \algname\ by sampling whole trajectories.
}
\label{alg:sampling-procedure}
\end{algorithm2e}

There is one subtlety that can introduce bias into
the simulated trajectories. What happens when the action
$\tilde{a}=\pi(\tilde{x},\tilde{w})$ is not equal to the
action $a$ in the database tuple $(\tilde{x}, \tilde{w}, a, r, x',w')$?
One approach would be to require that $a = \tilde{a}$ and keep rejecting candidate tuples until we find one that satisfies this constraint. We call this method, ``Biased MFMCi'', because doing this introduces a bias. Consider again the graphical model in Figure~\ref{fig:PGMMDP}. When we use the value of $a$ to decide whether to accept $\tilde{w}$, this couples $\tilde{w}$ and $\tilde{x}$ so that they are no longer independent. 

An alternative to Biased MFMCi is to change how we generate the database $D$ to ensure that for every state $(\tilde{x},\tilde{w})$, there is always a tuple $(\tilde{x}, \tilde{w}, a, r, x', w')$ for every possible action $a$. To do this, as we execute a trajectory following policy $\pi$, we simulate the result state $(x',w')$ and reward $r$ for each possible action $a$ and not just the action $a=\pi(x,w)$ dictated by the policy. We call this method ``Debiased MFMCi''. This requires drawing more samples during database construction, but it restores the independence of $\tilde{w}$ from $\tilde{x}$.

\subsection{MFMC with independencies (\algname)} \label{sec:MFMC}

For purposes of analyzing \algname, it is helpful to make the stochasticity of $P(s'|s,a)$ explicit. To do this, let $z$ be a time-independent random variable distributed according to $P(z)$. Then we can ``implement'' the stochastic transition $P(s'|s,a)$ in terms of a random draw of $z$ and a state transition {\it function} $f_x$ as follows. To make a state transition from state $s=(w,x)$ and action $a$, we draw samples of both the exogenous variable $w' \sim P(w')$ and from $z \sim P(z)$ and then evaluate the function $x' = f_x(x,w,a,z)$. The result state $s'$ is then $(x',w')$. Similarly, to model stochastic rewards,
we can define the function $f_r$ such that $r:= f_r(x,w,a,z)$.
This encapsulates all of the randomness in $P(s'|s,a)$ and $R(s,a)$ in the variables $w'$ and $z$. 

As noted in the previous section, stitching only on $x$ can
introduce bias unless we simulate the effect of every action $a$ for every $x\in D$. It is convenient to collect together all of these simulated successor states and rewards into {\it transition sets}. Let $B(x,w)$ denote the transition set of tuples $\{(x, w, a, x',r)\}$ generated by simulating each action $a$ in state $(x,w)$. Given a transition set $B$, it is useful to define selector notation as follows. Subscripts of $B$ constrain the set of matching tuples and superscripts indicate which variable is extracted from the matching tuples. Hence, $B_{a}^{x'}$ denotes the result state $x'$ for the tuple in $B$ that matches action $a$. With this notation, Algorithm~\ref{alg:sampling-procedure} describes the process of populating the database with transition sets.

\begin{algorithm2e}
\SetKwProg{Fn}{Function}{}{}
\DontPrintSemicolon
\textbf{Input Parameters:} Policy $\pi$, horizon $h$, starting state $x_0$, distance metric $\Delta_{i}(\cdot,\cdot)$, database $D$\;
\textbf{Returns:} $(x_0,w,a,r)_1,\ldots,(x',w',a',r')_h$\;
$t \gets \emptyset$\;
$x \gets x_0$\;
\While{length($t$)$<h$}{
  $\hat{B} \gets \argmin_{B \in D }\Delta_{i} (x, B^{x'})$\;
  $\hat{w} \gets \hat{B}^{w}$\;
  $a \gets \pi(x,\hat{w})$\;
  $r \gets \hat{B}^{r}_{a}$\;
  append$(t,(x,w,a,r))$\;
  $D \gets D\setminus \hat{B}$\;
  $x \gets B^{x'}_{a}$\;
}
return($t$)\;
\caption{\algname}
\label{alg:mfmci}
\end{algorithm2e}

Algorithm \ref{alg:mfmci} gives the pseudo-code for \algname.  Note
that when generating multiple trajectories with Algorithm
\ref{alg:mfmci} for a single policy query, the transition sets are drawn without replacement between trajectories.  To estimate the cumulative return of policy $\pi$, we call \algname\ $n$ times and compute the mean of the cumulative rewards of the resulting trajectories.  We refer to this as the MFMCi estimate $V_{MFMCi}^\pi(s_0)$ of $V^\pi(s_0)$.  

\subsection{Bias and Variance Bound on $V_{MFMCi}^\pi(s_0)$} \label{sec:mfmci}

Fonteneau et al.~\citeyearpar{Fonteneau2013,Fonteneau2014,Fonteneau2010c} derived bias and variance bounds on the MFMC value estimate $V_{MFMC}^\pi(s_0)$. Here we rework this derivation to provide analogous bounds for $V_{MFMCi}^\pi(s_0)$.  The Fonteneau et al.~bounds depend on assuming Lipschitz smoothness of the functions $f_s$, $f_r$ and the policy $\pi$. To do this, they require that the action space be continuous in a metric space $A$.  We will impose the same requirement for purposes of analysis.  Let $L_f$, $L_R$, and $L_{\pi}$ be Lipschitz constants
for the chosen norms $\|\cdot\|_S$ and $\|\cdot\|_A$ over the
$S$ and $A$ spaces, as follows:

\begin{equation} \label{eq:lsp}
 \|f_s(s,a,z)-f_s(s',a',z)\|_S~\leq~L_f(\|s-s'\|_S+\|a-a'\|_A) \\
\end{equation}

\begin{equation} \label{eq:lsr}
 |f_r(s,a,z)-f_r(s',a',z)|~\leq~L_R(\|s-s'\|_S+\|a-a'\|_A) \\
\end{equation}

\begin{equation} \label{eq:lspie}
 \|\pi(s)-\pi(s')\|_A\leq L_{\pi}(\|s-s'\|_S).\\
\end{equation}

To characterize the database's coverage of the state-action space, let $\alpha_{k}(D)$ be the maximum distance from any state-action pair $(s,a)$ to its $k$-th nearest neighbor in database $D$. Fonteneau, et al. call this the $k$-dispersion of the database.

\begin{theorem} \cite{Fonteneau2010c} For any Lipschitz continuous policy $\pi$, let 
$V_{MFMC}^{\pi}$ be the MFMC estimate of the value of $\pi$ in $s_0$ based on $n$ MFMC trajectories of length $h$ drawn from database $D$. Under the Lipschitz continuity assumptions of Equations \ref{eq:lsp}, \ref{eq:lsr}, and \ref{eq:lspie}, the bias and variance of $V_{MFMC}^{\pi}$ are 

  \begin{equation} |MFMC^{\pi}_r(s_0)-E^{\pi}_r(s_0)|\leq
  C\alpha_{nh}(D) \label{eq:mfmc_bias} \end{equation}

  \begin{equation}
  Var_{n,D}^{\pi}(s_0)\leq \Big(\frac{\sigma^\pi_h(s_0)}{\sqrt{n}}+2C\alpha_{nh}(D)\Big)^2
  \label{eq:variance}
  \end{equation}

where $\sigma^\pi_h(s_0)$ is the variance of the total reward for $h$-step trajectories under $\pi$ when executed on the true MDP and $C$ is defined in terms of the Lipschitz constants as

\begin{equation}
C = L_R \sum_{i=0}^{h-1} \sum_{j=0}^{h-i-1} [L_f (1 + L_{\pi})]^j.
\label{eq:mfmc_c}
\end{equation}
\end{theorem}

Now we derive analogous bias and variance bounds for $V_{MFMCi}^\pi(s_0)$. To this end, define two Lipschitz constants $L_{F_i}$ and $L_{R_i}$ such that the following conditions hold for the MDP:

\begin{equation} \label{eq:lspmfmci}
 \|f_x(x,a,w,z)-f_x(x',a,w,z)\|_X~\leq~ L_{f_i}(\|x-x'\|_X)
\end{equation}

\begin{equation} \label{eq:lsrmfmci}
 |f_r(x,a,w,z)-f_r(x',a,w,z)|~\leq~ L_{R_i}(\|x-x'\|_X).
\end{equation}

Where $\|\cdot\|_X$ is the chosen norm over the
$X$ space.

Let $\alpha_{i,k}(D)$ be the maximum distance from any Markov state
$x$ to its $k$-th nearest neighbor in database $D$ for the distance
metric $\Delta_i$. Further, let $x_0$ be the initial Markov state. Then we have

\begin{corollary}
For any Lipschitz continuous policy $\pi$, let 
$V_{MFMCi}^{\pi}$ be the MFMCi estimate of the value of $\pi$ in $x_0$ based on $n$ MFMCi trajectories of length $h$ drawn from database $D$. Under the Lipschitz continuity assumptions of Equations \ref{eq:lspmfmci} and \ref{eq:lsrmfmci}, the bias and variance of $V_{MFMCi}^{\pi}$ are 

\begin{equation} \label{eq:corbias}
|\algname^{\pi}(x_0)-E^{\pi}_r(x_0)|\leq C_i\alpha_{i,nh}(D)
\end{equation}

\begin{equation} \label{eq:corvariance}
Var_{i,n,D}^{\pi}(x_0)\leq \Big(\frac{\sigma_{h}^\pi(x_0)}{\sqrt{n}}+2C_i\alpha_{i,nh}(D)\Big)^2
\end{equation}

where $C_i$ is defined as

\begin{equation}
C_i = L_{R_i} \sum_{b=0}^{h-1} \sum_{j=0}^{h-b-1} [L_{f_i}]^j.
\label{eq:mfmci_c}
\end{equation}

\end{corollary}
\begin{proof} (Sketch) The result follows by observing that because there is always
a matching action for each transition set, $a$ will equal $a'$ and
$\|a-a'\|_A$ will be zero, so we can eliminate $L_{\pi}$.  Similarly,
because we can factor out $w$, we only match on $x$, so we can
replace $L_f$ with $L_{f_i}$ and replace the norms with respect to $S$
by the norms with respect to $X$. Finally, as we argued above, by
using transition sets we do not introduce any added bias by adopting
$w$ instead of matching against it. Formally, we can view this as
converting $w$ from being an observable exogenous variable to being part
of the unobserved exogenous source of stochasticity $z$. With these
changes, the proof of Fonteneau, et al., holds.
\end{proof}

We believe that similar proof techniques can bound the bias and variance of estimates of the quantiles of $P(v^j(s_t))$ for properties $v^j(s_t)$ of the state at time step $t$.  We leave this to future work.

\section{Experimental Evaluation} \label{sec:vis_wildfire}

In our experiments we test whether we can generate accurate trajectory
visualizations for a wildfire, timber, vegetation, and weather
simulator. The aim of the wildfire management simulator is to help US
Forest Service land managers decide whether suppress a wildfire on
National Forest lands. Each 100-year trajectory takes up to 7 hours to
simulate.

\begin{figure}
    \centering
        \includegraphics[width=0.90\columnwidth]{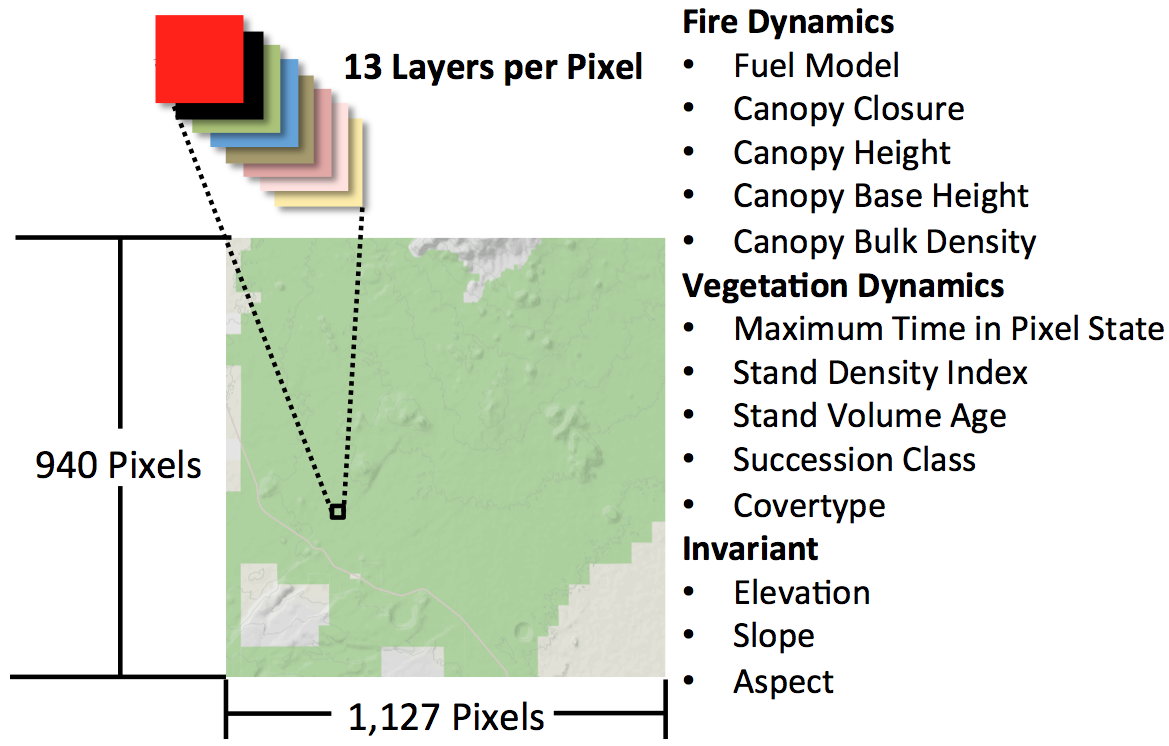}
        \caption{The landscape totals approximately one million pixels,
        each of which has 13 state variables that influence
        the spread of wildfire on the landscape.
        We use summary statistics of the dynamic state variables in MFMC's
        distance metric.
        (Map is copyright of OpenStreetMap contributors)}
        \label{fig:pixel-layers}
\end{figure}

Figure \ref{fig:pixel-layers} shows a snapshot of the landscape as
generated by the Houtman simulator \cite{Houtman2013}. The landscape is comprised of
approximately one million pixels, each with 13 state variables.
When a fire is ignited by lightning, the policy must choose between two actions:
\textit{Suppress} (fight the fire) and \textit{Let Burn} (do nothing). Hence, $|A|=2$.

The simulator spreads wildfires with the FARSITE fire model \cite{Finney1998}
according to the surrounding pixel variables ($X$) and the hourly weather.
Weather variables include \textit{
hourly wind speed,
hourly wind direction,
hourly cloud cover,
daily maximum/minimum temperature,
daily maximum/minimum humidity,
daily hour of maximum/minimum temperature,
daily precipitation, and
daily precipitation duration.}
These are generated by resampling from 25 years of observed weather
\cite{WesternRegionalClimateCenter2011}.
\algname\ can treat the weather variables and the ignition location
as exogenous variables because the decision to fight (or not fight)
a fire has no influence on weather
or ignition locations. Further, changes in the Markov state do not
influence the weather or the spatial probability of lightening strikes.

After computing the extent of the wildfire on the landscape,
the simulator applies a cached version of the Forest Vegetation
Simulator \cite{Dixon2002} to update the vegetation of the
individual pixels. Finally, a harvest scheduler selects
pixels to harvest for timber value.

We constructed three policy classes that map fires to fire suppression
decisions. We label these policies {\sc intensity}, {\sc fuel}, and
{\sc location}. The {\sc intensity} policy suppresses fires based on the
weather conditions at the time of the ignition and the number of days
remaining in the fire season.  The {\sc fuel} policy suppresses fires when
the landscape accumulates sufficient high-fuel pixels. The {\sc location}
policy suppresses fires starting on the top half of the landscape, and
allows fires on the bottom half of the landscape to burn (which mimics
the situation that arises when houses and other buildings occupy part
of the landscape). We selected these policy classes because they are
functions of different components of the Markov and exogenous state.
The {\sc intensity} policies are a function of the exogenous variables and
should be difficult for MFMC because the sequence of actions along a
trajectory will be driven primarily by the stochasticity of the
weather circumstances.  This contrasts with the {\sc fuel} policy, which
should follow the accumulation of vegetation between time steps in the
Markov state.  Finally, the {\sc location} policy should produce landscapes
that are very different from the other two policy classes as fuels
become spatially imbalanced in the Markov state.

The analysis of \citet{Fonteneau2013} assumes the database is
populated with state-action transitions covering the entire
state-action space. The dimensionality of the wildfire state space
makes it impossible to satisfy this assumption.  We focus sampling on states likely
to be entered by future policy queries by seeding the database with
one trajectory for each of 360 policies whose parameters are sampled
according to a grid over the {\sc intensity} policy space.  The {\sc intensity}
policy parameters include a measure of the weather conditions at the
time of ignition known as the Energy Release Component (ERC) and a measure
of the seasonal risk in the form of the calendar day. These measures
are drawn from $\big[0,100\big]$ and $\big[0,180\big]$, respectively.

The three policy classes are very different from each other.  One of our goals is to determine whether \algname\ can use state transitions generated from the {\sc intensity} policy to accurately simulate state transitions under the {\sc fuel} and {\sc location} policies. We evaluate \algname\ by generating 30 trajectories for each policy from the ground truth simulator.

For our distance metric $\Delta_i$, we use a weighted Euclidean
distance computed over the mean/variance standardized values of
the following landscape features:
\textit{%%Fuel Model,
Canopy Closure,
Canopy Height,
Canopy Base Height,
Canopy Bulk Density,
%%Covertype,
Stand Density Index,
High Fuel Count,
%%Succession Class,
%%Maximum Time in Pixel State,
and
Stand Volume Age}.
All of these variables are given a weight of 1.
An additional feature, the time step (\textit{Year}),
is added to the distance metric with a very large weight
to ensure that \algname\ will only stitch from one state
to another if the time steps match.
Introducing this non-stationarity ensures
we exactly capture landscape growth stages for all pixels
that do not experience fire.

Our choice of distance metric features
is motivated by the observation that
risk profile (the likely size of a wildfire)
and vegetation profile (the total tree cover)
are easy to capture in low dimensions. If we
instead attempt to capture the likely size of a
specific fire, we need a distance metric that
accounts for the exact spatial distribution of
fuels on the landscape. Our distance metric successfully
avoids directly modeling spatial complexity.

\begin{figure}
    \centering \includegraphics[width=0.75\columnwidth]{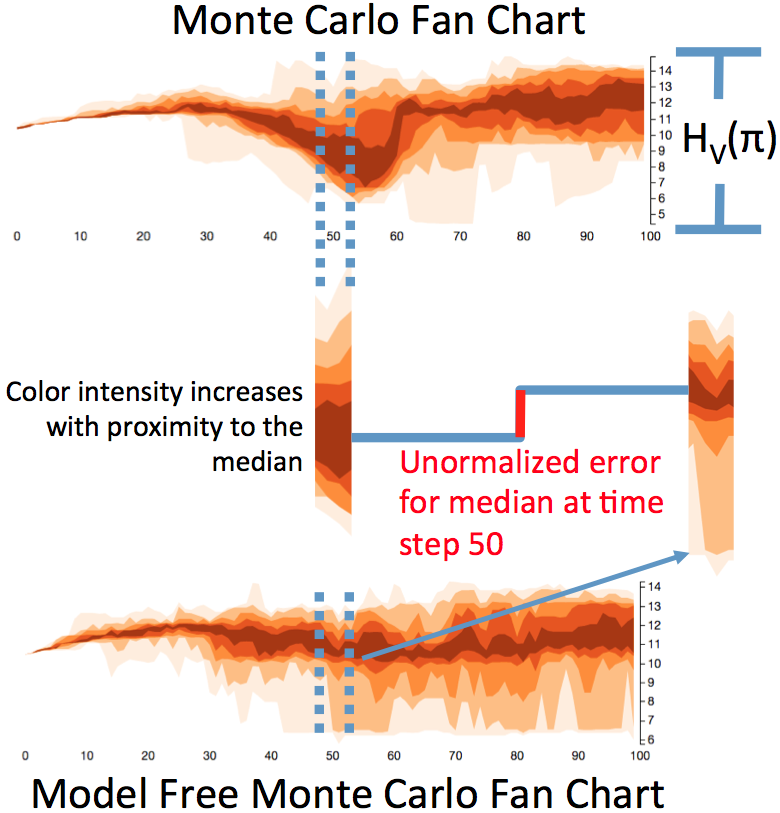} \caption{ Top: A fan chart generated by Monte Carlo simulations from the expensive simulator. Bottom:  A fan chart generated from the MFMC surrogate model. x axis is the time step and y axis is the value of the state variable at each time step.  Each change in color shows a quantile boundary for a set of trajectories generated under policy $\pi$. Middle: Error measure is the distance between the median of the Monte Carlo simulations (left) and the median of the MFMC/MFMCi surrogate simulations (right). The error is normalized across fan charts according to $H_{v}(\pi)$, which is the Monte Carlo fan chart height for policy $\pi$ and variable $v$. } \label{fig:fancharts}
\end{figure}

To visualize the trajectories, we employ the visualization
tool \visname{} \cite{McGregor2015a}.  The key visualization
in \visname{} is the fan chart, which depicts various quantiles of the
set of trajectories as a function of time (see
Figure \ref{fig:fancharts}).

To evaluate the quality of the fan charts generated using surrogate
trajectories, we define visual fidelity error in terms of the
difference in vertical position between the true median and its
position under the surrogate.  Specifically, we define
$\text{error}(v,t)$ as the offset between the correct location of the median and its MFMCi-modeled location
for state
variable $v$ in time step $t$. We normalize the error by the height of
the fan chart for the rendered policy ($H_{v}(\pi)$). The weighted
error is thus $\sum\limits_{v\in
S}\sum\limits_{t=0}^h\frac{\text{error}(v,t)}{H_v(\pi)}$.

This error is measured for 20 variables related to the counts of burned pixels, fire
suppression expenses, timber loss, timber harvest, and landscape
ecology.

\subsection{Experimental Results}

\captionsetup[subfigure]{skip=-8pt}

\begin{figure}
    \centering
    \begin{subfigure}[b]{1.0\columnwidth}
        \includegraphics[width=1.00\columnwidth]{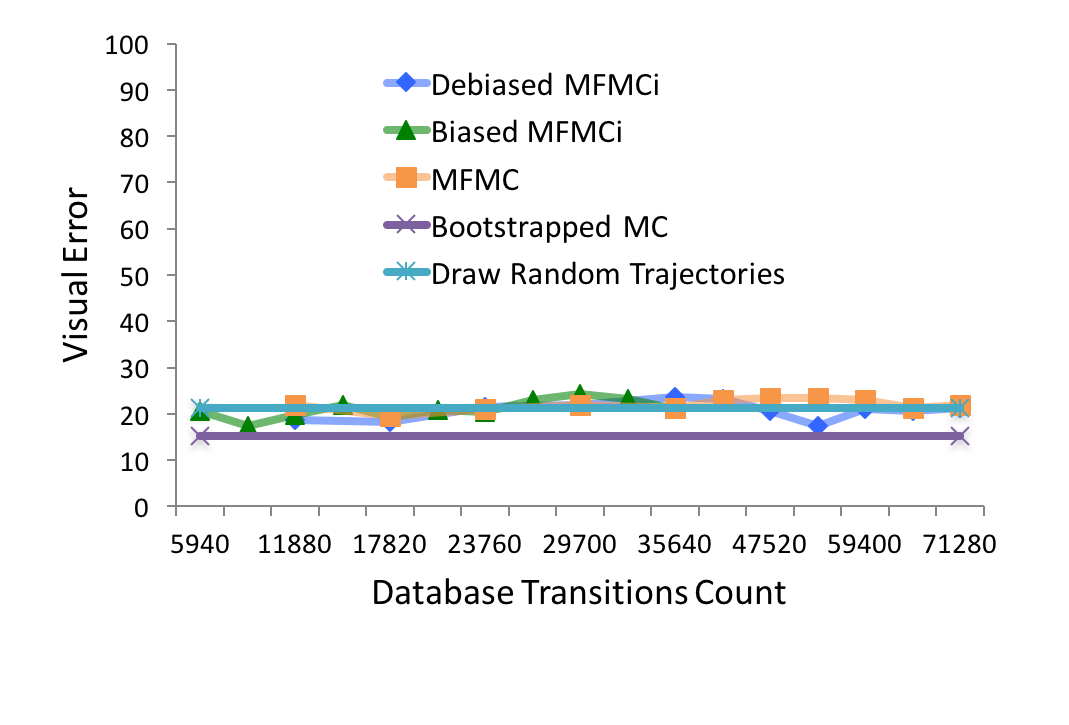}
        \caption{
        Visual fidelity errors for a weather \emph{intensity} policy class.
        Fires are suppressed based on a combination of the weather and how
        much time is left in the fire season.
        }
        \label{fig:intensity}
    \end{subfigure}
    \begin{subfigure}[b]{1.0\columnwidth}
        \includegraphics[width=1.00\columnwidth]{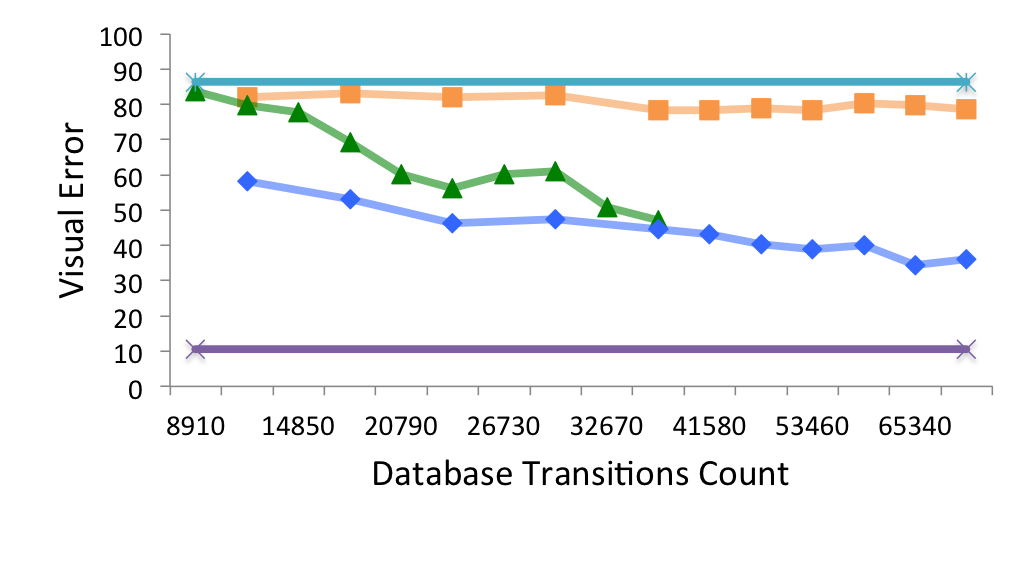}
        \caption{
        Visual fidelity errors for a ignition \emph{location} policy class.
        Fires are always suppressed if they start on the top half of the landscape, otherwise
        they are always allowed to burn.
        }
        \label{fig:location}
    \end{subfigure}
    \begin{subfigure}[b]{1.0\columnwidth}
        \includegraphics[width=1.00\columnwidth]{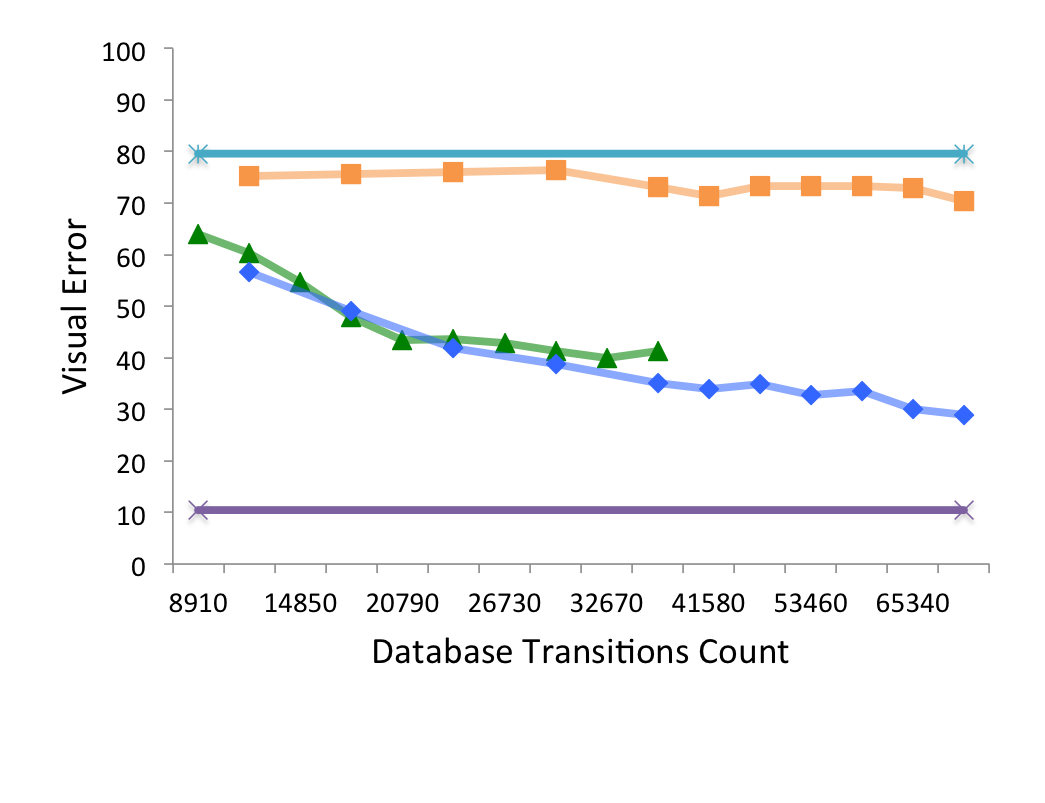}
        \caption{
        Visual fidelity errors for a \emph{fuel} accumulation policy class.
        Fires are always suppressed if the landscape is at least 30 percent in high fuels,
        otherwise the fire is allowed to burn.
        }
        \label{fig:biased}
    \end{subfigure}
    \caption{
      Policy classes for the wildfire domain under a variety of distance metrics and
      sampling procedures.
    }
    \label{fig:wildfire-quantitative}
\end{figure}

\captionsetup[subfigure]{skip=-1pt}

\begin{figure}
    \centering
        \includegraphics[width=.60\columnwidth]{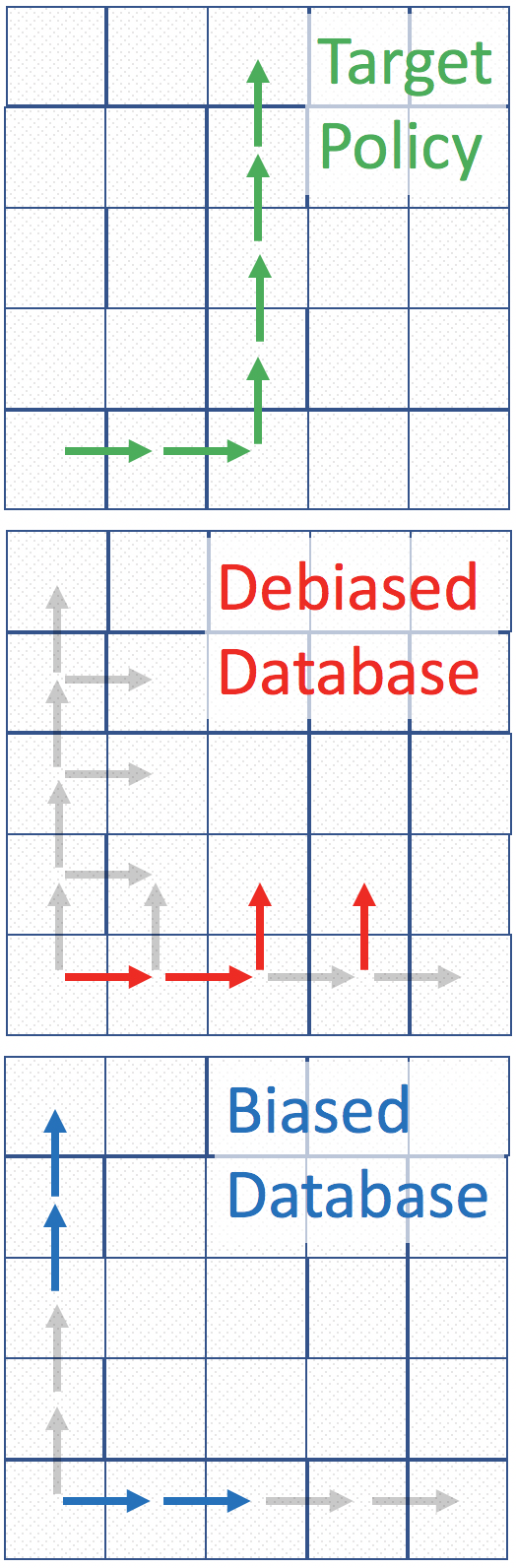}
        \caption{ Example of MFMC's autoregressive tendency
        for a grid world domain where the only available actions are
        ``up'' and ``right''. The green arrows show a trajectory
        that we would like to synthesize from two different MFMC databases
        where the distance metric is Euclidean with arbitrarily large weight
        given to the time step and action.
        The gray arrows show the grid world transitions in the
        two databases. In the debiased database the stitching
        operation will stay on the rightward trajectory despite
        there being transitions that more closely track the
        target trajectory. The biased database forces the stitching
        operation to hop to the policy more consistent with the target
        policy. In some instances it is better to bias
          the exogenous variables than repeatedly stitch to the same
          trajectories. }
        \label{fig:GridWorld}
\end{figure}

We evaluated the visual fidelity
under three settings:
(a) debiased MFMCi
(exogenous variables excluded from the distance metric $\Delta_i$;
debiasing tuples included in the database $D$),
(b) MFMC (exogenous variables included in $\Delta$), and
(c) biased MFMCi (exogenous variables excluded from $\Delta_i$
and the extra debiasing tuples removed from $D$).
We also compare against two baselines that explore the upper
and lower bounds of the visual error. First,
we show that the lower bound on visual error is not zero.
Although each policy has true quantile values at every time step,
estimating these quantiles with 30 trajectories is inherently
noisy. We estimate the
achievable visual fidelity by bootstrap resampling the 30 ground truth trajectories and
report the average visual fidelity error.
Second, we check whether the error introduced by stitching is worse than
visualizing a set of random database trajectories. Thus the bootstrap resample
forms a lower bound on the error, and comparison to the random trajectories
detects stitching failure.
Figure \ref{fig:wildfire-quantitative} plots ``learning curves'' that plot the visualization error as a function of the size of the database $D$. The ideal learning curve should show a rapid decrease in visual fidelity error as $|D|$ grows.

\section{Discussion}

For each policy class, we chose one target policy from that class and measured how well the behavior of that policy could be simulated by our MFMC variants. Recall that the database of transitions was generated using a range of {\sc intensity} policies.  When we apply the MFMC variants to generate trajectories for an {\sc intensity} policy, all methods (including random trajectory sampling) produce an accurate representation of the median for \visname{}.  When the database trajectories do not match the target policy, \algname\ outperforms MFMC. For some policies, the debiased database outperforms the biased databases, but the difference decreases with additional database samples. Next we explore these findings in more depth.

\textbf{{\sc Intensity} Policy.}
% erc threshold = 75 and days threshold = 120
% Suppresses approximately 60 percent of fires
Figure \ref{fig:intensity} shows the results of simulating
an {\sc intensity} policy that suppresses all fires
that have an ERC between 75 and 95,
and ignite after day 120.
This policy suppresses approximately 60 percent of fires.
There are many trajectories in the database that agree with the target
policy on the majority of fires. Thus, to simulate the target policy it is
sufficient to find a policy with a high level of agreement and then
sample the entire trajectory.
This is exactly what MFMC, MFMCi, and Biased \algname\ do.
All of them stitch to a good matching trajectory and then
follow it, so they all give accurate visualizations
as indicated by the low error rate in Figure \ref{fig:intensity}.
Unsurprisingly, we can approximate {\sc intensity} policies from
a very small database $D$ built from other {\sc intensity} policies.

\textbf{{\sc Location} Policy.}
Figure \ref{fig:location} plots the visual fidelity error when simulating a {\sc location} policy from the database of {\sc intensity} policy trajectories. When $D$ is small, the error is very high. MFMC is unable to reduce this error as $D$ grows, because its distance metric does not find matching fire conditions for similar landscapes. In contrast, because the MFMCi methods are matching on the smaller Markov state variables, they are able to find good matching trajectories. The debiased version of MFMCi outperforms the biased version for the smaller database sizes. In the biased version the matching operation repeatedly stitches over long distances to find a database trajectory with a matching action. Debiased MFMCi avoids this mistake. This explains why debiased MFMCi rapidly decreases the error while biased MFMCi takes a bit longer but then catches up at roughly $|D|=$40,000. 

\textbf{{\sc Fuel} Policy.}
% percentFuel > .30, then suppress
% Most wildfires after year 7 are suppressed
% https://github.com/smcgregor/FireWoman/blob/guide-for-logistic-policy/FireAp/Policy.cpp#L457
The {\sc fuel} policy shows a best case scenario for the biased database.  Within 7 time steps, fuel accumulation causes the policy action to switch from let-burn-all to suppress-all.  Since all of the trajectories in the database have a consistent probability of suppressing fires throughout all 100 years, the ideal algorithm will select a trajectory that allows all wildfires to burn for 7 years (to reduce fuel accumulation), then stitch to the most similar trajectory in year 8 that will suppress all future fires. The biased database will perform this ``policy switching'' by jumping between trajectories to find one that always performs an action consistent with the current policy.

Policy switching is preferable to the debiased database in some cases.  To illustrate this, consider the grid world example in Figure \ref{fig:GridWorld}. It shows that debiased samples can offer stitching opportunities that prevent policy switching and hurt the results.  

In summary, our experiments show that \algname{} is able to generalize across policy classes and that it requires only a small number of database trajectories to accurately reproduce the median of each state variable at each future time step. In general, it appears to be better to create a debiased database than a biased database having the same number of tuples. 

Our stakeholders in forestry plan to apply the \algname{} surrogate to their task of policy analysis.  We demonstrate the effectiveness of the \algname{} surrogate for the problem of wildfire policy optimization in \cite{McGregor2017}.

% todo: update the supporting materials bit

\section*{Acknowledgment}
  This material is based upon work supported by the National Science Foundation under Grant No. 1331932.

\bibliographystyle{icml2016}
\bibliography{references}

\end{document}